\documentclass[10pt,twocolumn,letterpaper]{article}

\usepackage{iccv}
\usepackage{times}
\usepackage{epsfig}
\usepackage{graphicx}
\usepackage{amsmath}
\usepackage{amssymb}

\usepackage[normalem]{ulem} 
\usepackage{algorithm}
\usepackage{algpseudocode}
\usepackage{booktabs}
\usepackage{caption}
\usepackage{subfigure}
\usepackage{enumitem}
\usepackage{dirtytalk}

\usepackage{amsthm}
\usepackage{algorithm}
\usepackage{algpseudocode}
\usepackage{booktabs}
\usepackage{caption}
\usepackage{subfigure}
\usepackage{amsmath}
\usepackage{amsthm}
 
\theoremstyle{definition}
\newtheorem{definition}{Definition}[section]
\newtheorem{proposition}{Proposition}

\date{}
 
\newcommand{\ke}[1]{{\color{black}#1}}
\newcommand{\etalke}{\textit{et al. }}
\newcommand{\ryn}[1]{\textcolor[rgb]{0,0,0}{#1}}

\newcommand{\daoye}[1]{\textcolor[rgb]{0,0,0}{#1}}
\newcommand{\yq}[1]{\textcolor[rgb]{0,0,0}{#1}}

\newcommand{\tabincell}[2]{\begin{tabular}{@{}#1@{}}#2\end{tabular}}

\usepackage[pagebackref=true,breaklinks=true,letterpaper=true,colorlinks,bookmarks=false]{hyperref}

\iccvfinalcopy % *** Uncomment this line for the final submission

\def\iccvPaperID{2205} % *** Enter the ICCV Paper ID here

\ificcvfinal\fi
\begin{document}

%%%%%%%%% TITLE
\title{Dual Student: Breaking the Limits of the Teacher in Semi-supervised Learning}

\author{
Zhanghan Ke\textsuperscript{1,2}
% \thanks{zhanghake2-c@my.cityu.edu.hk}
\thanks{kezhanghan@outlook.com}
\and Daoye Wang\textsuperscript{2} \and Qiong Yan\textsuperscript{2} \and Jimmy Ren\textsuperscript{2} \and Rynson W.H. Lau\textsuperscript{1} \\
% Institution1\\
% Institution1 address\\
% {\tt\small firstauthor@i1.org}
% For a paper whose authors are all at the same institution,
% omit the following lines up until the closing ``}''.
% Additional authors and addresses can be added with ``\and'',
% just like the second author.
% To save space, use either the email address or home page, not both
\and
% Second Author\\
\textsuperscript{1} City University of Hong Kong \and \textsuperscript{2} SenseTime Research \\
% \and 
% {\tt\small zhanghake2-c@my.cityu.edu.hk} 
% \and 
% {\tt\small \{wangdaoye, yanqiong, rensijie\}@sensetime.com} 
% \and
% {\tt\small rynson.lau@cityu.edu.hk} \\
}

\maketitle
%\thispagestyle{empty}

%%%%%%%% ABSTRACT
\begin{abstract}
    \ryn{Recently,} consistency-based methods have achieved state-of-the-art results
    in semi-supervised learning (SSL). These methods always \yq{involve} two roles, 
    an explicit or implicit teacher model and a student model, and \ryn{penalize} 
    predictions \yq{under} different perturbations by a consistency constraint. However, 
    the weights of these two roles are tightly coupled since the teacher is 
    essentially an exponential moving average (EMA) of the student. In this work, 
    we show that the coupled EMA teacher causes a performance bottleneck. 
    To address this problem, we introduce Dual Student, which 
    replaces the teacher with another student. We also define a novel concept, 
    stable sample, following which a stabilization constraint is designed for 
    our structure to be trainable. \ryn{Further, we discuss two variants of our method, which produce even higher performance}. Extensive experiments show 
    that our method 
    %\ke{\sout{could achieve state-of-the-art results}}
    \ke{\ryn{improves} the classification performance significantly} on several main SSL 
    benchmarks. Specifically, \ryn{it reduces} the error rate \ke{of the 13-layer CNN} from 16.84\% to 12.39\% 
    on CIFAR-10 with 1k labels and \ryn{from} 34.10\% to 31.56\% on CIFAR-100 with 10k labels. In addition, our method also \ryn{achieves} a clear improvement in 
    domain adaptation.
\end{abstract}
\vspace{-2mm}

\section{Introduction}
Deep supervised learning has gained significant success in computer vision tasks, 
which leads the community to challenge larger and more complicated datasets like 
ImageNet\,\cite{ImageNet} and WebVision\,\cite{WebVision}. However, obtaining full 
labels for a huge dataset is usually a \ryn{very costly task.} % difficult task due to various limitations. 
\ryn{Hence}, more attention is now drawn on deep semi-supervised 
learning (SSL). 
% Many methods have been proposed to utilize unlabeled data in 
% traditional machine learning\,\cite{SSL_Survey}, 
In order to utilize unlabeled data, many methods in traditional machine learning 
have been proposed\,\cite{SSL_Survey},
and some of them are %also 
successfully adapted \daoye{to} deep learning. \ryn{In addition}, some latest techniques like 
self-training\,\cite{TriNet} and Generative Adversarial Networks
(GANs)\,\cite{Un_Semi_GAN, SemiGAN, AuxiliaryGAN} have been utilized for deep SSL with 
promising results. A primary track of recent deep semi-supervised methods  
\cite{LadderNetwork, Temporal_Pi, AdvExample, MeanTeacher} can be summarized 
as consistency-based methods. In \ryn{this type of} methods, two roles are commonly created\ryn{, either explicitly or implicitly: a teacher model and a student model (i.e., a Teacher-Student structure).}
The teacher guides the student to approximate its 
performance under perturbations.
% , i.e., the ``teacher-student'' structure. 
The perturbations could come from the noise of \ryn{the} input or the dropout 
layer \,\cite{Dropout}, {\it etc.} \ryn{A consistency constraint is then} imposed on the 
predictions between two roles, \ryn{and forces the} unlabeled data to meet 
the {\it smoothness assumption} \ke{of semi-supervised learning}.
% In addition, the Smooth Neighbor \cite{SmoothNeighbor} 
% and the FastSWA \cite{FastSWA} enhance the smoothness by following the same principle.
% These works demonstrate the significance of smoothness to the SSL model.

% \begin{figure}[t]
% \centering
%     \includegraphics[width=0.7\linewidth]{img/result_show.pdf}
% \vspace{-0.2cm}
% \caption{
% \ryn{Comparison of our method with} a SSL milestone, the Mean Teacher (MT),
% on the CIFAR-10 SSL benchmark.
% % In this work, 
% We introduce Dual Student (DS) and its two variants: 
% Multiple Student (MS) and Imbalanced Student (IS). 
% \rynq{(*** I thought that we would change this figure to motivate the work (as stated in our rebuttal), instead of showing the performance. ***)}
% \ke{(KE: I am designing this figure now, I will finish it ASAP.)}
% }
% \vspace{-0.3cm}
% \label{fig:resutls}
% \end{figure}

\begin{figure}[t]
\centering
    \includegraphics[width=0.9\linewidth]{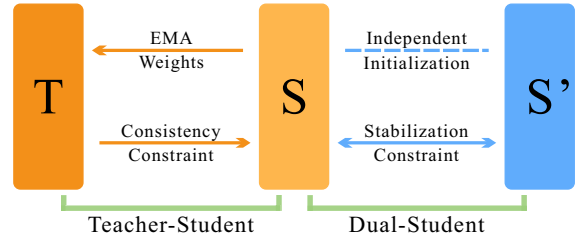}
\vspace{-0.2cm}
\caption{
% Illustration of previous Teacher-Student structure and our proposed Dual Student structure.
\ryn{Teacher-Student versus Dual Student. The teacher (T) in Teacher-Student is an EMA of the student (S), imposing a consistency constraint on the student. Their weights are tightly coupled. In contrast, a bidirectional stabilization constraint is applied between the two students (S and S') in Dual Student. Their weights are loosely coupled.}
%
%Teacher-Student: The teacher (T) is an EMA of the student (S) and forces consistency constraint to the student. Their weights are tightly coupled. Dual-Student: Two students (S and S') have different initial weights. A bidirectional stabilization constraint is applied between them. Their weights are loosely coupled.
}
\vspace{-0.3cm}
\label{fig:insight}
\end{figure}

The teacher in the Teacher-Student structure can be summarized as being generated by
an exponential moving average (EMA) of the student model.
In the VAT Model\,\cite{VAT} and the $\Pi$ Model\,\cite{Temporal_Pi},
the teacher shares the same weights as the student,
which is equivalent to \ryn{setting} the averaging coefficient to zero.
% However, the teacher is potentially generated from the student by the EMA.
% For example, the teacher shares the same weights as the student
% in VAT\,\cite{VAT} and $\Pi$ Model\,\cite{Temporal_Pi},
% which can be equivalent to an EMA with the coefficient zero. 
The Temporal Model \cite{Temporal_Pi} is similar to $\Pi$ Model
except that it also applies an EMA to accumulate the historical predictions. 
The Mean Teacher \cite{MeanTeacher} applies an EMA to the student to 
obtain an ensemble teacher. In this work, we show that the two roles 
in the Teacher-Student structure are tightly coupled and the degree of the
coupling increases as the training goes on. This phenomenon leads to a 
performance bottleneck 
% for the existing Teacher-Student 
% \yq{Teacher-Student or teacher-student, be consistent.} 
% methods 
since a coupled EMA teacher is not sufficient for the student.
% the targets from a coupled teacher are not sufficiently good for the student. 

To overcome this problem, the knowledge coming from another independent model should help. 
\ryn{Motivated by this observation}, we replace the EMA teacher by another student model. The two students start from different initial states 
% in the solution space
and are
optimized through individual paths during training. Hence, their weights will 
not be tightly coupled and each learns its own knowledge. What remains unclear is how to 
extract and exchange knowledge between the students.
% base on their different targets.
Naively, adding a consistency constraint \ryn{may lead to the two} models collapsing into each other. 
Thus, we define the {\it stable sample} and 
propose a stabilization constraint 
% based on them
% {\it stable samples} 
\ke{for} %\ke{\sout{to achieve}} 
effective knowledge \ryn{exchange}.
% In our experiments, 
Our method improves the performance \ke{significantly} %\ke{\sout{to a new state-of-the-art}} 
on several main SSL benchmarks. 
\ke{Fig.\,\ref{fig:insight} demonstrates the Teacher-Student structure and our Dual Student structure.}

In summary, the main contributions of this work include:
\begin{itemize}[itemsep=-2pt]
    \item We demonstrate that the coupled EMA teacher 
    causes a performance bottleneck of the existing Teacher-Student 
    methods.
    % method.
    \item We define the {\it stable samples} of a model and propose a novel 
    stabilization constraint between models. 
    % based on \ke{them}.
    \item We propose a new SSL structure, Dual Student, and \ryn{discuss two variants of Dual Student with higher performances}.
    \item Extensive experiments are conducted to evaluate the performance 
    of our method on several benchmarks and in different tasks.
\end{itemize}

\section {Related Work}

\subsection{Overview}
Consistency-based SSL methods are derived from the network noise 
regularization\,\cite{NoiseNorm}.
Goodfellow \etalke\,\cite{AdvExample} first showed the advantage of 
adversarial noise \ryn{over} random noise. 
Miyato \etalke\,\cite{VAT} further explored this idea for unlabeled 
data and generated virtual adversarial samples for the implicit 
teacher, while Park \etalke\,\cite{VDT} proposed a 
virtual adversarial dropout based on \cite{Dropout}. In addition 
to noise, the quality of targets for the consistency constraint is also vital 
in this process. Bachman \etalke\,\cite{PseudoEnsembles} and 
Rasmus \etalke\,\cite{LadderNetwork} showed the effectiveness of 
regularizing the targets. \ryn{Laine \etalke then proposed the} internally consistent 
$\Pi$ Model and Temporal Model in\,\cite{Temporal_Pi}. Tarvainen 
took advantage of averaging model weights\,\cite{WeightAverage} 
to obtain an explicit ensemble teacher \cite{MeanTeacher}
for generating targets.
Some works derived from the traditional methods also improve the 
consistency-based SSL. Smooth Neighbor by Luo \etalke\,\cite{SmoothNeighbor} 
utilized the connection between data points and built a neighbor graph to 
cluster data more tightly. 
% \yq{Should also desribe SWA here.}
Athiwaratkun \etalke\,\cite{FastSWA} modified 
the stochastic weight averaging (SWA)\,\cite{SWA} to \ryn{obtain} a stronger 
ensemble teacher faster. Qiao \etalke\,\cite{DeepCoTrain} proposed 
% a multiview algorithm, 
Deep Co-Training\,\cite{DeepCoTrain}, 
% by adding constraint between models. 
% which adds a constraint between models.
\ryn{by adding} a consistency constraint between independent models.
% as collaborators.

\subsection{Teacher-Student Structure}
The most common structure of recent SSL methods is the Teacher-Student structure. 
It applies a consistency constraint between a teacher model and a 
student model to learn knowledge from unlabeled data.
% between the two roles to meet the {\it smoothness 
% assumption}. A reliable teacher is key in this framework since all unlabeled 
% data only supervised by 
% the teacher's targets during training. Normally, performance improvements are primarily 
% made by adding better perturbation or constructing a better teacher. 
% Before continuing, let's define the problem and some notations first. 
% Assuming we are solving a classification problem in 
Formally, we assume that a dataset $\mathcal{D}$ consists of 
an unlabeled subset and a labeled subset. 
% an unlabeled set $\mathcal{U}$ and a labeled set $\mathcal{S}$.
% i.e., $\mathcal{D} = \mathcal{U} \cup \mathcal{S}$. 
Let $\theta^{'}$ denote \ryn{the}
weights of the teacher, and $\theta$ denote \ryn{the} weights of the student. 
The consistency constraint is defined as:
\begin{equation}\label{eq:cons_loss}
  \mathcal{L}_{con} = 
  \mathbb{E}_{x \in \mathcal{D}}\;\mathcal{R}(f(\theta, x+\zeta),\;\mathcal{T}_x)\,,
\end{equation}
where $f(\theta, x+\zeta)$ is the prediction from model $f(\theta)$ for input $x$ with 
noise $\zeta$. $\mathcal{T}_x$ is the consistency target from the teacher. 
$\mathcal{R}(\cdot,\,\cdot)$ measures the distance between two vectors, \ryn{and} is usually 
set to mean squared error (MSE) or KL-divergence. 
%Previous \daoye{\sout{works}} \daoye{work} have 
Previous works have proposed several ways to generate $\mathcal{T}_x$.

\textbf{$\Pi$ Model:}
In $\Pi$ Model, the implicit teacher shares parameters with the student. 
It forwards a sample $x$ twice with different random noise $\zeta$ and $\zeta'$ in 
each iteration, and treats the prediction of $x+\zeta'$ as $\mathcal{T}_x$.

\textbf{Temporal Model:}
While $\Pi$ Model needs to forward a sample twice in each iteration,  
\ryn{Temporal Model reduces this} computational overhead by using EMA to accumulate 
the predictions over epochs as $\mathcal{T}_x$. This approach could 
reduce the \ryn{prediction variance} and stabilize the training process. 

\textbf{Mean Teacher:}
Temporal Model needs to store a record for each sample\ryn{, and}
% , consuming huge memory.
% which will take a huge memory space when the dataset is enormous.
the target $\mathcal{T}_x$ gets updated only once per epoch while 
the student is updated multiple times.
% Thus, $\mathcal{T}_x$ is highly probably not good enough and 
% even stuck in a degenerated status. 
\ryn{Hence, Mean Teacher defines} an explicit teacher by an EMA of the student 
and update its weights in each iteration \yq{before generating} $\mathcal{T}_x$. 

\textbf{VAT Model:}
Although random noise is effective
% used in default 
in previous methods, VAT Model adopts 
the adversarial noise \ryn{to generate} better $\mathcal{T}_x$ for the consistency 
constraint.
% The noise $\zeta$ and $\zeta^{'}$ in all above methods are random.
% However, VAT Model extends adversarial samples to SSL and proves that a 
% teacher with adversarial noise $\zeta^{'}_{adv}$ could generate 
% better $\mathcal{T}_x$ for the consistency constraint. 

\subsection{Deep Co-Training}
It is known that fusing knowledge from multiple models could improve 
performance \ryn{in SSL \cite{SSLEnsemble}}. However, directly adding the 
consistency constraint between models results in \ryn{the} models collapsing 
into each other. 
% and stuck in a degraded status. 
Deep Co-Training addressed this issue \ryn{by} utilizing the {\it Co-Training assumption} 
from the traditional Co-Training algorithm\,\cite{CoTrain}. 
It treats the features from the convolutional layers as a view of the
input and uses the adversarial samples from other collaborators to ensure \ryn{that 
view differences exist} among the models. 
\ryn{Consistent predictions can then be used for training.} 
%\rynq{\ryn{Consistency on predictions can then} be used for training. (*** This sentence is unclear. It sounds like you are using the "consistency" for training, but I think that you are using the consistent predictions for training or you are using the consistency constraint in the training process? Please rewrite it if possible. ***)} \ke{(KE: This sentence means the consistency of Deep Co-Training, not our method.)}
However, this strategy requires \ryn{generating} adversarial samples of each 
model in the whole process, which is complicated and time-consuming.
% It treats the features from the convolutional layers as a 
% view of the input and the fully connected layers as a classifier.
% The models are trained simultaneously and their 
% outputs are constrained to follow the same distribution.
% \ke{
% To satisfy the {\it Co-Training assumption}, 
% Deep Co-Training has to generate adversarial samples 
% for its view difference constraint during training, 
% which is time consuming.}
% % creates another adversarial dataset by the generative method
% % and designs a view difference constraint. 
% % constraint \ke{to generate the different views of data}.
% % These two extra components make the whole structure difficult to train.

% \yq{The following should be place somewhere else...}
Our method also has interactions between models to break the limits of the 
EMA teacher, but there are two major differences between our method and 
Deep Co-Training. First, instead of \ryn{enforcing the consistency constraint \ke{and the different-views constraint},}
we only extract reliable knowledge of the models and exchange them by a more 
effective stabilization constraint. Second, our method is more efficient 
since we do not need the adversarial samples.
% Our idea is to break the limits of the EMA teacher by the interaction 
% between models. Different from Deep Co-Training, which exchanges all 
% knowledge of models by the consistency constraint, we think that the 
% reliable knowledge of the model has to be extracted first. Then we use 
% a unidirectional constraint to exchange these reliable knowledge. That 
% is, for each sample, the constraint is only applied from one model to 
% the other. Moreover, our method does not require the extra 
% adversarial samples. 

\section{Limits of the EMA Teacher}
One fundamental assumption in SSL is the {\it smoothness assumption - \say{If two data points 
in a high-density region are close, then so should be the corresponding outputs}} \cite{MIT_SSL}.
All existing Teacher-Student methods utilize unlabeled data according to 
this assumption. In practice, if $x$ and $\bar{x}$ are generated from a sample with 
different small perturbations,
they should have consistent predictions by the corresponding teacher and student. 
Previous methods achieve this by the consistency constraint and have mainly focused on generating more meaningful targets through ensemble or well-designed noise. 

However, previous works neglect that the teacher is essentially an EMA of the student\ryn{. Hence,} their 
weights are tightly coupled. Formally, the teacher weights $\theta^{'}$ are an 
ensemble of the student weights $\theta$ in a successive training step $t$ with a smoothing 
coefficient $\alpha \in [0, 1]$:
\begin{equation}\label{eq:ema_iter}
  \theta^{'}_{t} = \alpha\,\theta^{'}_{t-1} + (1 - \alpha)\,\theta_{t}\,.
\end{equation}
\ryn{In $\Pi$ Model and VAT Model, as $\alpha$ is set to zero, $\theta^{'}$ is equal to $\theta$.}
%\rynq{It is clear that $\theta^{'}$ equals to $\theta$ in the $\Pi$ Model and the VAT Model, where $\alpha$ equals zero. (*** Do you mean that $\theta^{'}$ equals to $\theta$ only when $\alpha$ equals zero? The current sentence is ambiguous on this. ***)} \ke{(KE: I want to mean ``When $\alpha$ equals zero, $\theta^{'}$ equals to $\theta$. And it is clear that the $\Pi$ Model and the VAT Model set $\alpha=0$.'')} 
Temporal Model improves $\Pi$ 
Model by an EMA on historical predictions, but its teacher still shares weights 
with the student. 
As for Mean Teacher, the updates of the student weights decreases as 
the model converges, i.e., $|\theta_t - \theta_{t-1}|$ becomes smaller and smaller  
as \ryn{the number of training steps} $t$ increases. Theoretically, it can be proved that the EMA of a converging sequence converges to the same limit as the sequence,  \ke{which is shown in Appendix A (Supplementary)}. 
Thus, the teacher will be very close to the student when 
the training process converges.
% By nature, the EMA gradually accumulate the latest knowledge 
% and throw away old information quickly. It means that the teacher weights 
% will move closer to the student. In other words, they will couple more 
% tightly. Notably, $\theta^{'}$ not differ much to $\theta$ in the 
% last stage.
In all the above cases, the coupling fact between the teacher and the student is obvious.

% \yq{Comment: is it ok to replace EMA Structure and Dual Structure with other 
% terms, like S-1, S-2. Since the names here are not quite accurate.}

To further visualize it, we train two structures on the CIFAR-10 SSL benchmark. 
One contains a student and an EMA teacher (named $S_{ema}$) while 
% is Mean Teacher without consistency constraint and 
the other contains two independent models (named $S_{split}$). \ryn{We then} 
calculate the Euclidean distance of \ryn{the} weights and
predictions between the two models in each structure. 
\ryn{Fig.\,\ref{fig:distance} shows the results.} 
As expected, 
% the EMA teacher is very close to the student 
the EMA teacher in $S_{ema}$ is very 
%\rynq{ (*** I personally think that it is better to say "very" here, as the difference is visible. ***)} \ke{(KE: I changed it to ``very'' again)} 
close to the student, and their distance approaches zero with increasing epochs. 
In contrast, 
the two models in $S_{split}$ always keep 
a %\ke{\sout{much}} 
larger distance from each other. 
% though it decreases.
These results confirm our conjecture that 
% the two roles in ``teacher-student'' are tightly coupled. 
the EMA teacher is tightly coupled with the student.
In addition, they also demonstrate 
that the two independent models are loosely coupled.
% in the contrary.

\begin{figure}[t]
\centering
    \includegraphics[width=0.99\linewidth]{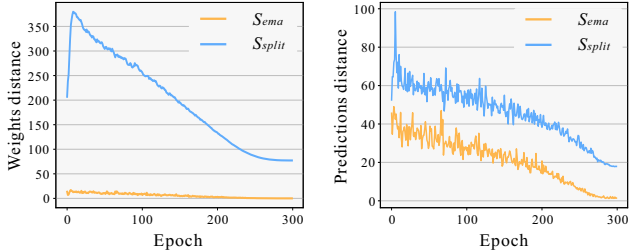}
    \vspace{-0.2cm}
\caption{Left: $S_{ema}$ contains two models with similar weights, 
         while \ryn{the} weights of the two models in $S_{split}$ keep a certain distance. 
         Right: \ryn{The predictions of the two models in $S_{split}$ 
         keep a larger distance than those of $S_{ema}$.} 
}
\vspace{-0.2cm}
\label{fig:distance}
\end{figure}

\begin{figure}[t]
\centering
    \includegraphics[width=0.7\linewidth]{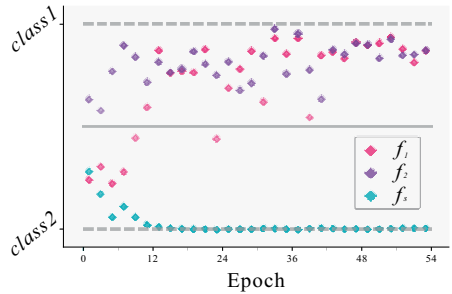}
\vspace{-0.2cm}
\caption{\ke{Our method can alleviate the confirmation bias.}
$f_1$ and $f_2$ are the independent students from our Dual Student, while  
$f_s$ is the student guided by the Mean Teacher. For a misclassified sample (\ryn{belonging} to {\it class1}), 
$f_1$ can correct it quickly with the knowledge from $f_2$. \ke{However, $f_s$ is unable to correct its prediction due to the wrong guidance from the EMA teacher.}
}
\vspace{-4mm}
\label{fig:samples}
\end{figure}

\begin{figure*}[t]
    \setlength{\abovecaptionskip}{-0.2cm}
    \begin{center}
       \includegraphics[width=0.85\linewidth]{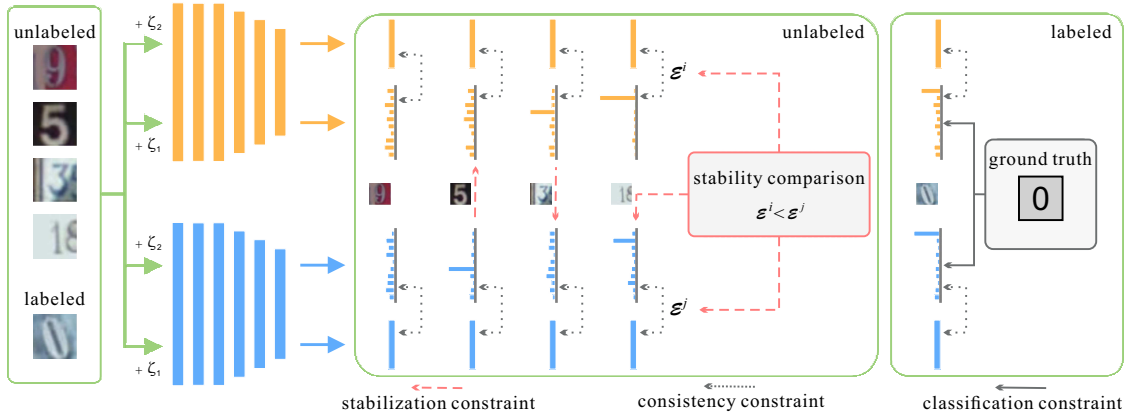}
    \end{center}
    \caption{Dual Student structure overview. We train two student models separately. 
    Each batch includes labeled and unlabeled data and is forwarded twice. The stabilization 
    constraint based on the {\it stable samples} is enforced between the students. Each student also learns 
    labeled data by the classification constraint and meets the {\it smooth assumption} by the consistency 
    constraint.
     }
     \vspace{-3mm}
    \label{fig:framework}
\end{figure*}

\ryn{Due to the coupling effect between the two roles in the existing Teacher-Student methods, the teacher} does not have more meaningful knowledge compared to the student.
% , which 
% means that the structure cannot make good use of the unlabeled data. 
\ryn{In addition}, if the student 
has \daoye{biased} predictions for specific samples, the EMA teacher is most likely to accumulate the 
mistakes and to enforce the student to follow, making the misclassification irreversible. 
\ryn{This} is a case of the confirmation bias \cite{MeanTeacher}. 
% Most methods apply a ramp up 
% operation for consistency constraint to alleviate the bias. Nonetheless, this trick 
% is inadequate due to the tight coupling between the two roles. From this perspective, 
% training independent models are also beneficial, because they may make diverse 
% predictions and search a broader solution space. 
Most methods apply a ramp-up operation for the consistency constraint to alleviate the bias,
but it is inadequate to solve the problem. 
From this perspective, 
training independent models are also beneficial.
% because they may make diverse 
% predictions and search a broader solution space. 
% Ramp-up operation for consistency constraint is proposed yet 
% inadequate to solved the problem. 
% \ke{Therefore, we} resolve to the potential of two student models. 
% Furthermore, 
\ryn{Fig.\,\ref{fig:samples} visualizes this inability of the EMA teacher.}
Three models, $f_1$, $f_2$, and $f_s$, are trained on a two-category task simultaneously.
% $f_1$, $f_2$ and $f_s$. \ke{trained on a two-category task simultaneously.}
$f_s$ is the student from Mean Teacher. 
$f_1$ and $f_2$ are two relatively independent but interactive models, \ryn{representing} 
the two students from our Dual Student structure (\ryn{Section}\,\ref{sec:dual_student}).
\ryn{They} have the same initialization, while $f_2$ is different from them.
The plot shows how the predictions of a sample from {\it class1} changes with 
epochs for these three models, 
\ke{which demonstrates that our method can alleviate the confirmation bias.}
% \ke{\sout{
% At the beginning, both $f_1$ and $f_s$ are wrong. 
% $f_s$ is unable to correct the prediction due to the wrong guidance from the EMA teacher, exhibiting the above mentioned performance issue. On the contrary, $f_1$ corrects \ke{\sout{this sample}} \ke{the mistake} quickly with the knowledge from $f_2$. }}

% For visualization, we train three models $M_1$, $M_2$ and $M_M$ simultaneously in 
% a two-category task. Here $M_M$ is  Mean Teacher, while $M_1$ and $M_2$ are 
% two independent but interactive models. For a fair comparison, $M_1$ and $M_M$ have the same 
% initialization while the initialization of $M_2$ is different. Then we 
% track the predictions of unlabeled data after each epoch. Fig.\,\ref{fig:samples} 
% shows a sample belonging to {\it class1}, but all models made a mistake to recognize 
% it as {\it class2} at the beginning. It is impossible for $M_M$ to correct it due 
% to a wrong guidance from the EMA teacher. On the contrary, $M_1$ can correct this 
% sample quickly with the help of $M_2$. 

\section{Dual Student}\label{sec:dual_student}
As analyzed above, the targets from an EMA teacher are not adequate to guide the student 
when \ryn{the number of training steps} $t$ is large. 
% Our idea is to
Therefore, our method gains the loosely coupled targets by training two 
independent models simultaneously. 
However, \ryn{the} outputs of these two models may vary 
widely, and applying the consistency constraint directly will \ryn{cause them to} collapse into each other by exchanging the wrong knowledge. 
The EMA teacher does not suffer from this problem 
%\ke{\sout{since it has a similar optimization path with the student}}
\ke{due to the coupling effect.}
% Deep Co-Training \cite{DeepCoTrain} addresses such a learning problem by fitting
% the \ke{different views conditions} of Co-Training \cite{CoTrain}. However, 
% such structure is difficult to train.
% ----------------------------------------------
% the extra generative model and the views difference constraint make 
% Hence it assembles a generative model to each model and applies a views 
% difference constraint, which makes the whole pipeline complicated and difficult 
% to train. 

We propose an efficient way to overcome this problem, 
% We propose a more efficient way to train multiple models, 
which is to \ryn{exchange only} reliable knowledge of \ryn{the} models. 
% In practice, 
To put this idea into practice,
\ryn{we need to solve two problems.}
%\yq{\sout{there are}} two problems \ryn{needed} to be solved. 
One is how to define and acquire reliable knowledge of a model. Another 
is how to exchange the knowledge mutually. 
To address them, we define the {\it stable sample} in Section \ref{sec:stable_sample}
and then elaborate the derived stabilization constraint for training in Section \ref{sec:stable_constraint}.
% For the first problem, we define the {\it stable samples} of a model based on the 
% {\it smoothness assumption} and treat them as the reliable knowledge of a model. 
% For the second problem, we utilize these {\it stable samples} to impose a stabilization 
% constraint between models.

% and design a stabilization constraint to 
% acquire and exchange the  knowledge in models level. Our method shows 
% that when these knowledge be exchanged appropriately, both the generate models and 
% {\it views difference prior} are unnecessary. 
% In this section we first define the {\it stabilization assumption} to gain the
% stabilization samples, then we discuss how to utilize them impose the 
% stabilization constraint between models.   

\subsection{Stable Sample}\label{sec:stable_sample}
A model can be regarded as a decision function that can make reliable 
predictions for some 
samples but not for the others. We define the {\it stable sample} and treat it 
as the reliable knowledge of a model. 
% the samples with confident predictions, i.e., 
% the reliable knowledge of the model. 
A {\it stable sample} satisfies two conditions. First, according to the 
{\it smoothness assumption}, a small perturbation should not affect the 
prediction of this sample, 
% it indicates that 
i.e., the model should be smooth in the neighborhood of 
this sample. 
% Second, the predicted label of this sample should have a high probability,
% since a sample whose prediction is near the classification boundary is unreliable. 
\yq{Second, the prediction of this sample \ryn{is far} from the decision boundary\ryn{. This means that this sample has a high probability for the predicted label}.}

% \yq{Do not use $f$ here to avoid ambiguity with previous notation.}
\theoremstyle{definition}\label{def:stable_samples}
\begin{definition}[{\it Stable sample}]
Given a constant $\xi \in [0,1)$, 
a dataset $\mathcal{D}\subseteq\mathbb{R}^m$ that satisfies the {\it smoothness assumption} and a model $f:\mathcal{D}\to[0,1]^n$ that satisfies $||f(x)||_{1}=1$ for all  
$x\,\in\,\mathcal{D}$, $x$ is a {\it stable sample} with respect to $f$ if:
% In a set of data points $\mathcal{D}\,\subseteq\,\mathbb{R}^m$,
% a model $f(\theta):\,\mathcal{D}\,\to\,[0,\,1]^n$ 
% meets $||\,f(\theta,\,x)\,||_{1}\,=\,1$, $\forall x\,\in\,\mathcal{D}$
% In a dataset  
% $\mathcal{D}\,\subseteq\,\mathbb{R}^m$, 
% Given $\xi \in [0, 1]$ and a set of data points $\mathcal{D}\,\subseteq\,\mathbb{R}^m$.
% $ \forall\,x\,\in\,\mathcal{D}$, a model $f(\theta):\,\mathcal{D}\,\to\,[0,\,1]^n$ 
% meets $||\,f(\theta,\,x)\,||_{1}\,=\,1$. The $x$ is stable corresponding to $f(\theta)$ if:
% $\zeta\,\sim\,N(0,\,\eta\,I_m)$, $\eta > 0$ is a small constant, 
% $I_m$ is the identity matrix.
% Given a set of data points $\mathcal{D}\,\subseteq\,\mathbb{R}^m$ 
% and a model $f(\theta):\,\mathcal{D}\,\to\,[0,\,1]^n$.
% $ \forall\,x\,\in\,\mathcal{D}$, $||\,f(\theta,\,x)\,||_{1}\,=\,1$. 
% $||\,f(\theta, x)\,||_{1} = 1$.
% $\xi \in [0, 1]$ is a threshold.
% A data point $x \in \mathcal{X}$ is stable corresponding to $f(\theta)$ if:
\begin{enumerate}
\itemsep0em 
    \item $\forall \bar{x} \in \mathcal{D}$ near $x$, their predicted labels are the same.
    % \item if $\bar{x} \in \mathcal{D}$ is near $x$, their predicted labels are the same.
    % ----------------------------------------------------------------------
    % \item $\forall\bar{x} \in \mathcal{D}$,
    % if $\bar{x}$ is near $x$, 
    % their predicted labels are the same.
    % ------------------------------------
    % should have the same prediction as $x$ if they are close. 
    % $\bar{x}$ and $x$ have same label if they are close
    % predicted labels are same if they are
    % if $\bar{x}$ is close to $x$, 
    %$\underset{i}{\arg\min} (f(\theta, \bar{x}))_{i} = \underset{i}{\arg\min} (f(\theta, x))_{i}$
    % $f(\theta, \bar{x})$ has the same prediction as $f(\theta, x)$, 
    % where $\zeta$ is a small perturbation.
    \item $x$ satisfies the inequality: $||f(x)||_{\infty}  > \xi$\;. \footnote{\footnotesize
    $||\,\mathbf{a}\,||_{1} := \sum\limits_{i=1}^n |\,a_i\,| $,\;
    $||\,\mathbf{a}\,||_{\infty} := \max\limits_{i=1..n} |\,a_i\,| $,\; 
    $\mathbf{a} = (a_1, a_2, ..., a_n)$ }

    % \,, where $\xi$ is a confidence threshold 
    % from the decision boundary to $f(\theta, x)$. 
\end{enumerate}
\end{definition}
\ryn{Def.\,\ref{def:stable_samples} defines the {\it stable sample}, and 
Fig.\,\ref{fig:assumption} illustrates its conditions} in details.
Notice that the concept of the {\it stable sample} is specific to \ryn{the} models.
A data point $x$ can be stable with respect to \ryn{any one} model
%\rynq{one model (*** Do you mean "one pair of models" here? ***)} \ke{(KE: No. ``one model'' just means ``any model''. For example, assuming we have 3 models, a data point $x$ can be stable with respect to any one of 3 models while it is unstable to other 2 models.)} 
but may not \ryn{be to the} others. This fact is a key to our stabilization constraint, \ryn{and will be elaborated in Section \ref{sec:stable_constraint}}.
% An example are shown in Fig. \ref{fig:assumption}.
\ryn{In addition to the criterion of} whether a sample point $x$ is stable 
or not, we would \ryn{also} like to know the degree of stability of a 
{\it stable sample} $x$. This can be reflected by \ryn{the} prediction consistency in its 
neighborhood. The more consistent the predictions are, the more stable $x$ is.

\subsection{Training by \ryn{the} Stabilization Constraint}\label{sec:stable_constraint}
We briefly introduce Dual Student structure before explaining the details on training.
\ryn{It contains two independent student models, which} share the same network architecture with different initial states and are updated 
% simultaneously but 
separately (Fig.\,\ref{fig:framework}).
% Fig.\,\ref{fig:framework} shows our Dual Student structure. 
% We call this the Dual Student. 
% As we analyzed above,
% % Comparing with the EMA teacher, 
% two independent models 
% % keep a large distance in both 
% % weights and predictions, i.e., they 
% are loosely coupling and have their own knowledge. 
% Whereas, it also means their outputs may vary widely, applying the 
% consistency constraint let them collapse into each other by 
% exchanging the wrong knowledge. However, the EMA teacher does not suffer 
% from this problem since it has a similar optimize path with the student.
\ryn{For our structure to be} trainable, we derive a novel stabilization constraint from
the {\it stable sample}.
% In order to train it, 
% we propose a novel stabilization constraint.
% % to overcome such a learning problem.
% This constraint is derived from the {\it stable samples}, which are defined 
% in Def.\,\ref{def:stable_samples}. 
% In practice, to evaluate whether a sample is stable we use only two close 
% samples to approximate the first conditions of the {\it stable sample} to  
% reduce computational cost.

In practice, we only utilize two close samples to approximate the conditions of 
the {\it stable sample} to reduce \ryn{the} computational overhead.
% more samples will increase the computational overhead. 
Formally, we use $\theta^{i}$ \ryn{and $\theta^j$ to represent weights of the two students. \ryn{We first} define a boolean function ${\{condition\}}_{\textbf{1}}$, which outputs 1 
when the condition is true and 0 otherwise.} 
Suppose $\bar{x}$ is a noisy augmentation of a sample $x$. 
\ryn{We then check whether $x$ is a {\it stable sample} for student $i$}:
\begin{equation}\label{eq:stabilization_judgement}
\begin{split}
    &\mathcal{R}^{i}_x=
    {\{ \mathcal{P}^{i}_{x} = \mathcal{P}^{i}_{\bar{x}} \}}_{\textbf{1}}\,
    \&\,
    ({\{\mathcal{M}^i_{x} > \xi \}}_{\textbf{1}}
    \|\,
    {\{\mathcal{M}^i_{\bar{x}}  > \xi \}}_{\textbf{1}})\;, \\
&\textrm{where} 
 %\qquad\qquad
 \qquad \qquad \mathcal{M}^i_x = ||\,f(\theta^i,\,x)\,||_{\infty}.
    % \mathcal{R}^{i}(x)\,= 
    % {\{ P^{i}_{x_1} = P^{i}_{x_2} \}}_{\textbf{1}}\;
    % \&\;
    % (\,{\{ \max(s(f^{i}(x_{1}))) > \xi \}}_{\textbf{1}}\;
    % \|\;
    % {\{ \max(s(f^{i}(x_{2}))) > \xi \}}_{\textbf{1}}\,)\;
    % {\textbf{1}}_{\{ \prod^2_{m=0}(\max(softmax(f^{i}(x_{1}))) > \xi) \}}\; \\
    % \mathbf{\&}\;\prod^2_{m=0}(\max(softm/ax(f^{i}(x_{1}))) > \xi) \\
\end{split}
\end{equation}
$\mathcal{P}^{i}_x$ \ryn{and} $\mathcal{P}^{i}_{\bar{x}}$ 
are the predicted labels of $x$ \ryn{and $\bar{x}$, respectively, by} student $i$. 
Hyperparameter $\xi$ is a confidence threshold in $[0, 1)$. \ryn{If the maximum prediction probability of sample $x$ exceeds $\xi$, $x$} is considered to be far enough 
from the classification boundary. \ryn{We then} use the Euclidean distance to measure the prediction consistency\ryn{,} to \yq{indicate} the stability of $x$\ryn{,} \ryn{as}: 
\begin{equation}\label{eq:stabilization_distance}
    \mathcal{E}^{i}_x\,=\,||\,f(\theta^{i},\,x) - f(\theta^{i},\,\bar{x})\,||^2\;.
\end{equation}
\ryn{A smaller $\mathcal{E}^{i}_x$ means that $x$ is more stable to student $i$.}
%The smaller $\mathcal{E}^{i}_x$ means \ryn{that} \rynq{student $i$ is more stable for $x$ (*** Do you mean "$x$ is more stable for student $i$"? ***)} \ke{(KE: I think that it should be ``a model stable for a sample $x$''.)}.
\ryn{The distance between the predictions of students $i$ and $j$ can be measured using the mean squared error (MSE) as:}
%Let $\mathcal {L}_{mse}$ denote the mean squared error (MSE) between the students' predictions: 
\begin{equation}\label{eq:stabilization_distance}
    \mathcal{L}_{mse} (x)\,=\,||\,f(\theta^{i},\,x) - f(\theta^{j},\,x)\,||^2\;.
\end{equation}
%\rynq{(*** Why do we need to define MSE without a reason? The next sentence seems to be on a different point as you start with "Finally". I suppose that the next sentence is continuous from the last sentence. Please rewrite. ***)}
%\ke{(KE: Above special MSE is between two models, i.e., model $i$ and $j$, so we have to define it first.)}
%\rynq{Finally, (*** Does this correspond to "First" above? ****)} \ke{(KE: Yes, it correspond to "First", "Second", "Then" above)} 
Finally, the stabilization constraint for the student $i$ on sample $x$ is written as:
% \begin{equation}\label{eq:mse}
% L_{mse}(x)\,=\,||\,s(f^{i}(x_{1})) - s(f^{j}(x_{2}))\,||^2
% \end{equation}
% Finally, our stabilization constraint for model $i$ in sample $x$ is:
% \begin{equation}\label{eq:stabilization_constraint}
% \begin{split}
%   L^{i}_{sta}(x) = \left\{ 
%         \begin{matrix} {\{ D^{i}(x) < D^{j}(x) \}}_{\textbf{1}}\; L_{mse}(x), & S^{i}(x) = S^{j}(X) = 1, \\
%         S^{i}(x)\; L_{mse}(x), & \textrm{otherwise. \quad\quad\quad\quad\quad\quad\quad} \end{matrix}   \right.
% \end{split}
% \end{equation}
\begin{equation}\label{eq:stabilization_constraint}
\mathcal{L}^{i}_{sta}(x) = \\
 \left\{\begin{aligned}
        &{\{ \mathcal{E}^{i}_x > \mathcal{E}^{j}_x \}}_{\textbf{1}}\; \mathcal{L}_{mse}(x),
        \quad \mathcal{R}^{i}_x = \mathcal{R}^{j}_x = 1, \\
        &\; \mathcal{R}^{j}_x\; \mathcal{L}_{mse}(x), \quad\quad\quad\quad\; \textrm{otherwise.}
       \end{aligned}
 \right.
\end{equation}
\ke{We calculate the stabilization constraint for the student $j$ in the same way.}
As we can see, the stabilization constraint changes 
dynamically depending on the outputs of the two students. There 
are three cases: 
(1) No constraint is applied if $x$ is unstable for both students.
(2) If $x$ is stable only for student $i$, it \ryn{can guide} student $j$. 
(3) \ryn{If} $x$ is stable for both students, the stability is 
calculated, and the constraint is applied from the more stable one to the other. 

\begin{figure}[t]
    \setlength{\abovecaptionskip}{-0.2cm}
    \begin{center}
       \includegraphics[width=0.7\linewidth]{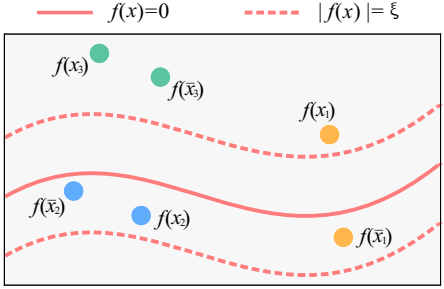}
    \end{center}
    \caption{
    Illustration of the \ryn{conditions for a {\it stable sample}}.
    % Illustration of the stable conditions through 
    \ryn{Consider} three pairs of adjacent data points:
             (1) \ryn{$x_1$ and $\bar{x}_1$ do not satisfy the 1\textsuperscript{st} condition, 
             (2) $x_2$ and $\bar{x}_2$ do not satisfy the 2\textsuperscript{nd} condition, and
             (3) $x_3$ and $\bar{x}_3$ satisfy} both conditions.}
    \label{fig:assumption}
    % \vspace{-2mm}
\end{figure}

% Other part of our framework
\yq{Following previous works}, our Dual Student structure also imposes the consistency constraint \ryn{in} each student 
to meet the {\it smoothness assumption}. \ryn{We also apply} the decoupled top layers trick from the Mean Teacher, which splits 
the constraints for the classification and the smoothness.

To train Dual Student, the final constraint for student $i$ is 
a combination of three parts\ryn{: the classification constraint, 
consistency constraint \ke{in each model}, and stabilization constraint \ke{between models}, as:} 
\begin{equation}\label{eq:loss}
  \mathcal{L}^i = \mathcal{L}^{i}_{cls} + \lambda_{1} \; \mathcal{L}^{i}_{con} + \lambda_{2}\; \mathcal{L}^{i}_{sta}\;,
\end{equation}
where $\lambda_{1}$ and $\lambda_{2}$ are hyperparameters to balance 
the constraints. \ryn{Algorithm\,\ref{alg::DS_SSL} summarizes the optimization process.} 
% In each iteration, we sample both labeled and 
% unlabeled samples as a mini batch and input them to each student twice with 
% different augmentations. The labeled samples are learned from the ground 
% truth by the classification constraint while the unlabeled samples are 
% learned from the {\it stable samples} by the stabilization constraint. 
% Besides, all samples tend to meet the {\it smoothness assumption} through 
% the consistency constraint inside the model.

\begin{algorithm}[t]
  \caption{Training of Dual Student for SSL.}
  \label{alg::DS_SSL}
  \begin{small}
  \begin{algorithmic}[1]
    \Require
        \ryn{Batch $\mathcal{B}$ containing} labeled and unlabeled samples
        % input batch $\mathcal{B} = \{\mathcal{X}_\mathcal{U},\,\mathcal{X}_\mathcal{S}\}$
    \Require
        \ryn{Two} independent models $f(\theta)$ and $f(\theta^{'})$ 

    \For {each batch $\mathcal{B}$}
            % \State Generate augmented $\mathcal{B}_1$, $\mathcal{B}_2$ from $\mathcal{B}$
            \State Get $\mathcal{B}_1$, $\mathcal{B}_2$ from $\mathcal{B}$ by data \ryn{augmentation}
            \For {each model in \{$f(\theta)$, $f(\theta^{'})$\}}
                % \State Calculate $\mathcal{L}_{cls}$ on $\mathcal{X}_\mathcal{S}$
                \State Calculate $\mathcal{L}_{cls}$ on labeled samples
                \State Calculate $\mathcal{L}_{con}$ by Eq.\,\ryn{\ref{eq:cons_loss}} between $\mathcal{B}_1$ \ryn{and} $\mathcal{B}_2$
            \EndFor
            % \For {each sample $x$ in $\mathcal{X}_\mathcal{U}$}
            \For {each unlabeled sample $x$}
                \For {each model in \{$f(\theta)$, $f(\theta^{'})$\}}
                    \State Determine whether $x$ is stable by Eq.\,\ryn{\ref{eq:stabilization_judgement}}
                \EndFor
                \If {both $f(\theta)$ and $f(\theta^{'})$ are stable for $x$}
                    \State Calculate the stability of $x$ by Eq.\,\ryn{\ref{eq:stabilization_distance}}
                \EndIf
                \State Calculate $\mathcal{L}_{sta}$ for $f(\theta)$ and $f(\theta^{'})$ by Eq.\,\ryn{\ref{eq:stabilization_constraint}}
            \EndFor
    \State Update $f(\theta)$ and $f(\theta^{'})$ by the loss in Eq.\,\ryn{\ref{eq:loss}}
    \EndFor
  \end{algorithmic}
  \end{small}
\end{algorithm}

\begin{table*}[t]
  \begin{center}
    \caption{Test error rate on CIFAR-10 averaged over 5 runs. 
             \ryn{Parentheses show numbers of training epochs (default 300).}}
    \label{tab:CIFAR-10}
    \begin{tabular}{lllll}
      \toprule 
      \textbf{Model} & 1k labels & 2k labels & 4k labels & all labels \\
      \midrule
      $\Pi$\,\cite{Temporal_Pi} & $31.65\pm1.20^{\dagger}$ & $17.57\pm0.44^{\dagger}$ & $12.36\pm0.31$ & $5.56\pm0.10$ \\
      $\Pi$\,+\,SN\,\cite{SmoothNeighbor} & $21.23\pm1.27$ & $14.65\pm0.31$ & $11.00\pm0.13$ & $5.19\pm0.14$  \\
      \hline
      Temp\,\cite{Temporal_Pi} & $23.31\pm1.01^{\dagger}$ & $15.64\pm0.39^{\dagger}$ & $12.16\pm0.24$ & $5.60\pm0.10$ \\
      Temp\,+\,SN\,\cite{SmoothNeighbor} & $18.41\pm0.52$ & $13.64\pm0.32$ & $10.93\pm0.34$ & $5.20\pm0.14$  \\
      \hline
      MT\,\cite{MeanTeacher} & $18.78\pm0.31^{\dagger}$ & $14.43\pm0.20^{\dagger}$ & $11.41\pm0.27^{\dagger}$ & $5.98\pm0.21^{\dagger}$ \\
      MT\,+\,FSWA\,\cite{FastSWA} & $16.84\pm0.62$ & $12.24\pm0.31$ & $9.86\pm0.27$ & $\mathbf{5.14\pm0.07}$ \\
      \hline 
      CS & $17.38 \pm 0.52$ & $13.76 \pm 0.27$ & $10.24 \pm 0.20$ & $5.18 \pm 0.11$ \\    
      DS & $\mathbf{15.74 \pm 0.45}$ & $\mathbf{11.47\pm0.14}$ & $\mathbf{9.65\pm0.12}$ & $5.20\pm0.03$ \\    
      \hline                
      MT\,+\,FSWA (1200)\,\cite{FastSWA} & $15.58\pm0.12$ & $11.02\pm0.23$ & $9.05\pm0.21$ & $4.73\pm0.18$ \\
      Deep CT (600)\,\cite{DeepCoTrain} & - & - & $9.03\pm0.18$ & - \\
      DS (600) & $\mathbf{14.17\pm0.38}$ & $\mathbf{10.72\pm0.19}$ & $\mathbf{8.89\pm0.09}$ & $\mathbf{4.66\pm0.07}$ \\
      \bottomrule
    \end{tabular}
  \end{center}
  \vspace{-0.2cm}
\end{table*}

\subsection{Variants of Dual Student}
\ryn{Here, we briefly discuss two variants of Dual Student, named \emph{Multiple Student} and \emph{Imbalanced Student}}. Both of them 
have higher \ryn{performances} than the standard Dual Student. 
% \ke{\sout{\ryn{However,}}} 
They do not increase the inference time, even though 
more computations are required during training. 
% They do not 
% require extra resource in reference, but only more computation 

\textbf{Multiple Student:}
Our Dual Student \ryn{can be easily extended to} Multiple Student. 
We followed the same strategy as the Deep Co-Training.
% The Dual Student is a special case of Multiple Student, which 
% interactive multiple models simultaneously. We follow the training 
% strategy in Deep Co-Training. 
\ryn{We assume that} our Multiple Student contains $2n$ student models. 
At each iteration, we randomly divide these students into $n$ pairs. 
\ryn{Each pair is then updated like} Dual Student. 
\ryn{Since our method does not require models to have view differences,}
%Since there is no view difference requirement in our method,
the data stream can be shared among \ryn{the students. This is different from Deep 
Co-Training, which} requires an exclusive data stream for each pair.
% Our Multiple Student is more computationally efficient since all 
% models share the same data stream while the Deep Co-Training train 
% each pair by an exclusive data stream. 
In practice, four students \yq{($n=2$)} achieve a notable improvement over two students. 
However, \ryn{having more than four students do not further improve the performance, as demonstrated in Section~\ref{sec:perf_variants}.} 
%\rynq{$<$--- Is it possible to explain why? ***)} \ke{(KE: This phenomenon is shown by experiments. We cannot explain it here.)}
% \yq{Comment: what is point? Four model improves a lot, but eight/sixteen 
% has much more advantage?}

\textbf{Imbalanced Student:}
% \yq{
% We have assumed the two models of Dual Student to be the same. 
% \ke{The two models of Dual Student share the same architecture. As we know, }
% However,
\ryn{Since a well-designed architecture with more parameters usually has better performance, 
% Knowledge distillation \cite{Distill, DistillTS} use it to distill knowledge 
% from a pre-trained high-performance complex model 
% to improve another light-weight model.
% ----------------------------------------
% \rynq{
% Motivated by it, 
% \ryn{knowledge} distillation \cite{Distill, DistillTS} improves the light-weight student 
% by the knowledge distilled from a pre-trained high-performance teacher.
% (*** The meaning of this sentence is ambiguous. ***)}
% \ke{(KE: This sentence explains knowledge distillation. In knowledge distillation, knowledge from a large pre-trained model (teacher) is used to help to train a small model (student).)}
% ----------------------------------------
a pre-trained high-performance teacher can be used} to improve the light-weight student in knowledge distillation task \cite{Distill, DistillTS}.
% However, such a teacher does not exist in SSL due to the lack of labels.
% Based on it, the Knowledge Distillation task \cite{Distill, DistillTS} train a 
% complex model first. Then the knowledge distilled from it is utilized to improve 
% the performance of another simple model. 
% However, its difficult to train a complex model in SSL due to the lack of labels. 
% However, we have assumed the two models of Dual Student to be the same. 
\ke{Based on the same idea,}
%\ke{\sout{To test how a strong-weak pair of models behaves in our structure,}} 
we extend \ke{Dual Student} to Imbalanced Student by enhancing the capability of one student.
% The strong student not only plays a role in SSL, but also transfers the knowledge to the weak one.
% Both models are considered as the students 
\ke{However, we} do not consider the sophisticated model as a teacher, since the knowledge will be exchanged mutually.
We \ryn{find} that the improvement of the weak student is proportional to the capability of the \daoye{strong} student. 

\section{Experiments}
We first evaluate Dual Student on several common SSL benchmarks, including  
CIFAR, SVHN, and ImageNet. \ryn{We then evaluate the performances of the two variants
of Dual Student}.
We \ryn{further} analyze various aspects of our method \ryn{through} ablation experiments. 
Finally, we \ryn{demonstrate the application of Dual Student in a} domain adaptation task.

\ryn{Unless specified otherwise}, the architecture used in our experiments is \ryn{a} same 13-layer convolutional neural network \ke{(CNN)}, following previous \ryn{works}\,\cite{Temporal_Pi, VAT, MeanTeacher}. 
Its details are described in \ke{Appendix B (Supplementary)}. As reported in\,\cite{SSLEval}, the implementations of recent SSL methods are not exactly same, and the training 
details 
(\ryn{e.g., number of training epochs, optimizer and augmentation}) \ryn{may also be different.} 
For a fair comparison, we implement our method following the previous state-of-the-art \cite{FastSWA}, which uses the 
standard Batch Norm\,\cite{BatchNorm} instead of the mean-only Batch Norm\,\cite{MeanOnlyBN}. 
The stochastic gradient descent optimizer is adopted with the learning rate 
adjustment function $\gamma = \gamma_0 * (0.5 + \cos((t - 1) * \pi / N))$, where 
$t$ is the current training step, $N$ is the total \ryn{number of} steps, and $\gamma_0$ is the initial learning rate. These \ryn{settings provide better baselines for} $\Pi$ Model and Mean Teacher. 
% as reported by \cite{FastSWA}. 
For other methods, we use the results from the original papers. 
More training details are provided in \ke{Appendix C (Supplementary)}.

\subsection{SSL Benchmarks}
\ryn{We first} evaluate Dual Student on the CIFAR benchmark, including CIFAR-10\,\cite{CIFAR-10} 
and CIFAR-100\,\cite{CIFAR-100}. CIFAR-10 has 50k training samples and 10k testing samples, \ryn{from} 10 categories. Each sample is a $32\times32$ RGB image. 
% Following previous works, 
We extract 1k, 2k, and 4k balanced labels randomly. 
CIFAR-100\,\cite{CIFAR-100} is a more complex dataset \ryn{including} 100 categories.
Each \ryn{category} contains only 500 training samples\ryn{, together with 100 test} samples. We extract 
10k balanced labels from it randomly. Besides, we also run experiments 
with full labels on both datasets. We compare our Dual Student (DS) with some recent 
consistency-based models, including $\Pi$ Model ($\Pi$), Temporal Model (Temp), 
Mean Teacher (MT), Smooth Neighbor (SN), FastSWA based on Mean Teacher (MT+FSWA), and Deep Co-Training (Deep CT). 
We also replace the stabilization constraint in our structure with 
the consistency constraint (CS) as a baseline.

\begin{table}[t]
  \small
  \begin{center}
    \caption{Test error rate on CIFAR-100 averaged over 5 runs.}
    \label{tab:CIFAR-100}
     \begin{tabular}{lll}
      \toprule 
      \textbf{Model} & 10k labels & all labels \\
      \midrule
      Temp\,\cite{Temporal_Pi} & $38.65\pm0.51$ & $26.30\pm0.15$ \\
      $\Pi$\,\cite{Temporal_Pi} & $39.19\pm0.36$ & $26.32\pm0.04$ \\
      $\Pi$\,+\,FSWA\,\cite{FastSWA} & $35.14 \pm 0.71$ & $22.00\pm0.21$ \\
      \hline
      MT\,\cite{MeanTeacher} & $35.96\pm0.77^{\dagger}$ & $23.37\pm0.16^{\dagger}$ \\
      MT\,+\,FSWA\,\cite{FastSWA} & $34.10\pm0.31$ & $\mathbf{21.84\pm0.12}$ \\
     \hline
    %   CS & - & - \\
      DS & $\mathbf{33.08\pm0.27}$ & $21.90\pm0.14$ \\    
      \hline
      MT\,+\,FSWA (1200)\,\cite{FastSWA} & $33.62\pm0.54$ & $\mathbf{21.52\pm0.12}$ \\
      Deep CT (600)\,\cite{DeepCoTrain} & $34.63\pm0.14$ & - \\
      DS (480) & $\mathbf{32.77\pm0.24}$ & $21.79\pm0.11$ \\    
      \bottomrule
    \end{tabular}
  \end{center}
  \vspace{-0.4cm}
\end{table}

\ryn{Table\;\ref{tab:CIFAR-10} shows the results on CIFAR-10. All models are trained for 300 epochs, except for those specified with parentheses.
% Note that most models are trained for about 300 epochs, otherwise the suffix 
% inside the parenthesis indicates the actual number of epochs. 
Results marked with a ${\dagger}$ are obtained from other works that published better performances than the original ones.
We can see that our} 
Dual Student boosts the performance on all semi-supervised settings. 
The results reveal that as the number of 
labeled samples decreases, our method can gain more significant improvements.
Specifically, Dual Student improves the result with 1k labels to $14.17\%$ with 
only \ryn{half of} training epochs comparing to \ke{FastSWA.} 
%\rynq{\sout{FastSWA} (*** Which one exactly are you referring to in Table 1? Are you referring to "MT+FSWA (1200)"? ***)} \ke{(KE: Yes. MT+FSWA is previous SOTA. The performance of ``DS (600)'' is better than ``MT+FSWA (1200)'')}. 
% When evaluated in standard train time, our result, $15.58\%$, is also the best.
Similar results \ryn{can also be observed in the experiments with 2k and 4k labels.}
% Specifically, Dual Student improves the result to $15.74\%$ with 1k labels, comparable 
% with the previous best result $15.58\%$ from FastSWA, which is trained for 3 times more epochs. 
% Moreover, Dual student achieves $14.08\%$ with longer training.
% the poor results of CS also prove 
% the effectiveness of our stabilization constraint
% it only exchanges the relatively reliable knowledge between models.
Fig.\,\ref{fig:stable_improve} shows that the accuracy on only the {\it stable samples} is 
higher than that on all samples, 
which proves that the {\it stable samples} represent the relatively more
reliable knowledge of a model. This justifies 
why our DS with stabilization constraint achieves much better results than the CS. 
% [We compare the results with Deep Co-Training in 4k labels experiments since they only 
% report this one.] can be removed
Our result on full labels shows less advantages since 
the labels play a much more important role
in the fully supervised case. 
% Our result $8.89\%$ is better than $9.03\%$ from the 
% two views Deep Co-Training. 
Table\;\ref{tab:CIFAR-100} lists the results on CIFAR-100. 
% The performance gain agrees with Table\;\ref{tab:CIFAR-10}. 
Especially, in 10k label experiments, 
Dual Student records a new state-of-the-art $32.77\%$ with less training epochs
than FastSWA and Deep Co-Training.
 % achieves $33.41\%$, and improves it to $32.81\%$ with a longer training. 
% In this case, our approach has a clear advantage over Deep Co-Training.
% In addition, our approach also yield good results on the experiments with full labels.

\begin{figure}[t]
    \setlength{\abovecaptionskip}{-0.2cm}
    \begin{center}
       \includegraphics[width=0.8\linewidth]{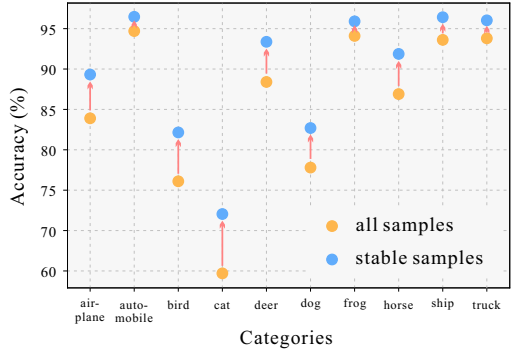}
    \end{center}
    \caption{
    Test accuracy of each category on the {\it stable samples} \ryn{and on 
    all samples of} CIFAR-10. The performance gap indicates that the {\it stable samples} represent relatively more reliable knowledge of a model. The 
    average ratio of the {\it stable samples}  on the test set is about 85\% w.r.t. the model.
    }
    \label{fig:stable_improve}
\end{figure}

\begin{table}[t]
  \begin{center}
    \caption{Test error rate on SVHN averaged over 5 runs.}
    \label{tab:svhn}
    \begin{tabular}{lllll}
      \toprule 
      \textbf{Model}                     & 250 labels               & 500 labels                \\
      \midrule
      Supervised\,\cite{MeanTeacher}      & $27.77 \pm 3.18$         & $16.88 \pm 1.30$          \\
      MT \,\cite{MeanTeacher}             & $4.35 \pm 0.50$          &  $4.18 \pm 0.27$          \\
    %   MT + SN\,\cite{SmoothNeighbor}      & $4.29 \pm 0.23$          &  $3.99 \pm 0.24$          \\
      DS                           & $\mathbf{4.24 \pm 0.10}$ & $\mathbf{3.96 \pm 0.15}$        \\    
      \bottomrule
    \end{tabular}
  \end{center}
  \vspace{-0.2cm}
\end{table}

\begin{table}[t]
  \begin{center}
    \caption{Test error rate on ImageNet averaged over 2 runs.}
    \label{tab:imagenet}
    \begin{tabular}{lll}
      \toprule 
      \textbf{Model}                    & 10\% labels-top1            & 10\% labels-top5            \\
      \midrule
      Supervised                        & $42.15 \pm 0.09$ & $19.76 \pm 0.11$ \\
      MT\,\cite{MeanTeacher}            & $37.83 \pm 0.12$          & $16.65 \pm 0.08$          \\
      DS                                & $\mathbf{36.48 \pm 0.05}$ & $\mathbf{16.42 \pm 0.07}$ \\
      \bottomrule
    \end{tabular}
  \end{center}
  \vspace{-0.2cm}
\end{table}

To evaluate the generalization ability of Dual Student, we also conduct 
experiments \ryn{on both SVHN\,\cite{SVHN} and ImageNet\,\cite{ImageNet}}.
Street View House Numbers (SVHN) is a dataset containing 73,257 training 
samples and 26,032 testing samples. Each sample is a $32\times32$ 
RGB image with a center close-up of \ryn{a} house number. We only experiment  
with 250 and 500 labels on SVHN. ImageNet contains more than 10 million RGB 
images belonging to 1k categories. We extract 10\% balanced labels and train 
a 50-layer ResNeXt model \cite{ResNeXt}. 
\ryn{Tables\;\ref{tab:svhn} and \ref{tab:imagenet}} show 
that Dual Student could improve the results on \ryn{these} datasets of various scales.
% on various scales datasets.
% still boost the result even the test error is pretty low in a simple dataset
% The predict value for each class is pretty small due to many categories.
% in ImageNet, the weights of consistency constraint and stabilization constraint need to be larger, 
% and the stabilization threshold need to be small since a higher value will totally disable the 
% knowledge exchanging in the early stage. The results are listed in Table\;\ref{tab:imagenet}.  

\begin{table}[t]
  \begin{center}
    \caption{
    Test error rate of two variants of Dual Student (all using the 13-layer CNN) on 
    % 13-layer CNN's test error rate of two variants of Dual Student on 
    the CIFAR benchmark averaged over 3 runs.
    \ryn{Parentheses of Multiple Student (MS) indicate the numbers of students. Parentheses of Imbalanced Student (IS) indicate the numbers of parameters for the strong} student.}
    \label{tab:variants}
    \begin{tabular}{lll}
      \toprule 
      \textbf{Model} & \tabincell{c}{CIFAR-10\\ 1k labels} & \tabincell{c}{CIFAR-100\\ 10k labels}\\
      \midrule
      DS                           & $15.74 \pm 0.45$ & $33.08 \pm 0.27$        \\    
      \hline
      MS (4 models) & $14.97 \pm 0.36$ & $32.89 \pm 0.32$ \\
      MS (8 models) & $14.77 \pm 0.33$ & $32.83 \pm 0.28$                   \\
      \hline
      IS (3.53M params) & $13.43 \pm 0.24$ & $32.59 \pm 0.27$ \\
      IS (11.6M params) & $\mathbf{12.39 \pm 0.26}$ & $\mathbf{31.56 \pm 0.22}$ \\
      \bottomrule
    \end{tabular}
  \end{center}
\end{table}

\begin{figure}[t]
\centering
    \includegraphics[width=0.99\linewidth]{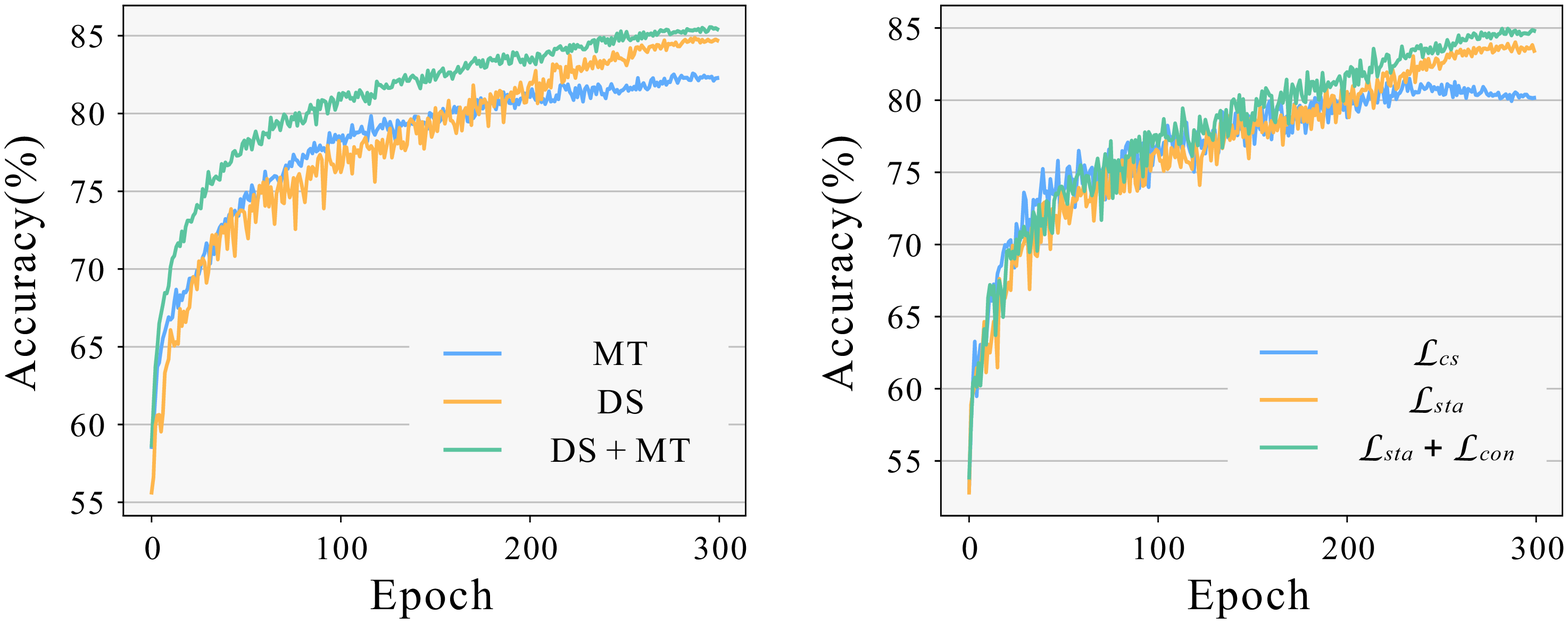}
\caption{Test accuracy on CIFAR-10 with 1k labels.
         Left: Combining our method with Mean Teacher can improve its performance. 
         Right: The effectiveness of our stabilization constraint.}
\label{fig:ablation}
\vspace{-0.2cm}
\end{figure}

\subsection{Performance of Variants}\label{sec:perf_variants}
We evaluate Multiple Student and Imbalanced Student on the CIFAR benchmark. 
Table\,\ref{tab:variants} compares them with the standard Dual Student\ryn{, 
% All results are 
all using} the same 13-layer CNN trained for 300 epochs. 
For Multiple Student (MS), 
we train both the four students and the eight students. The performance improvement is limited 
when more than four students are trained simultaneously. For Imbalanced Student (IS), we 
replace one student by a ResNet \cite{ResNet} with Shake-Shake regularization. \ryn{We then 
conduct the} experiments on two different model sizes. In particular,  
a small one with 3.53 million parameters and a 
large one with 11.65 million parameters. The small ResNet has almost no increase in 
\ryn{computational cost, as its number of parameters is similar to that of} the 13-layer CNN (3.13 million parameters). Imbalanced Student achieves a significant performance improvement by distilling the knowledge from a more powerful student. Notably, the large ResNet 
improves the result from 15.74\% to 12.39\% on CIFAR-10 with 1k labels.

\ryn{Our} structure can also be combined with existing methods easily to \ryn{further improve the performance}. We replace the consistency constraint inside the model 
by Mean Teacher. Fig.\,\ref{fig:ablation} (left) \ryn{shows} the accuracy 
curves. The obvious \ryn{performance improvement shows} the ability of Dual Student in breaking the limits of the EMA teacher. The accuracy of the combination 
is similar to that using Dual Student only, which means that our method 
is insensitive to the type of consistency constraint inside each model.

\subsection{Ablation Experiments}
We conduct \ryn{the} ablation experiments on CIFAR-10 with 1k labels to 
% analyze two aspects of Dual Student, the impact
analyze the impact 
of the confidence threshold and \yq{various} constraints in our structure.
%, and the compatibility with the Men Teacher. 

{\bf Confidence threshold:} 
\ke{
The confidence threshold $\xi$ controls the 2\textsuperscript{nd} condition in Def.\,\ref{def:stable_samples} \daoye{of} the {\it stable sample} by filtering out samples near to the boundary. Its \ryn{actual value can be set approximately,} since our method is robust to it.
Typically, $\xi$ is related to the complexity of the task, e.g., the number of categories to predict or the size of the given dataset. More categories or \yq{a} smaller size would require a smaller $\xi$. Table\;\ref{tab:xi_values} compares \ryn{different $\xi$ values on the} CIFAR 
benchmark. The results \ryn{show} that $\xi$ is necessary for a better performance\ryn{, and
a meticulous tuning may only} help improve the performance slightly.
}

%\ke{\sout{
%The confidence threshold $\xi$ is a hyperparameter in our structure. It controls the stabilization constraint by filtering out samples near to the boundary. In a simple dataset, the model always makes high probability predictions. A low $\xi$ will treat the unreliable knowledge as the reliable one. On the contrary, the prediction probability is relatively low in a complex dataset. A high $\xi$ will almost eliminate the constraint between models. when $\xi = 0$, i.e., no 2\textsuperscript{nd} condition in Def.\,\ref{def:stable_samples} any more, the error rate is 16.49\% on CIFAR-10 with 1k labels. It indicates that the improvement can be up to about 0.8\% when $\xi$ is appropriately set. In addition, we also find out that Dual Student has some tolerance for $\xi$, meaning that a slight change does not affect the error rate much.
%}}

{\bf Effect of \ryn{the} constraints:}
Dual Student learns the unlabeled data by both $\mathcal{L}_{sta}$ 
between models and the $\mathcal{L}_{con}$ inside each model. We also study \ryn{their individual impacts}. Besides, we compare the results with 
the experiment where only  the consistency constraint is applied
% applied the consistency constraint 
between models 
(named $\mathcal{L}_{cs}$). Fig.\;\ref{fig:ablation} (right) shows 
that $\mathcal{L}_{cs}$ \ryn{reduces the} accuracy in the 
late stage while $\mathcal{L}_{sta}$ \ryn{helps improve the} performance continuously.
This demonstrates that our $\mathcal{L}_{sta}$ \ryn{is better than} $\mathcal{L}_{cs}$. \ryn{In addition,} $\mathcal{L}_{con}$ \ke{inside the model} also plays a 
role in boosting the performance further.

\begin{table}[t]
  \small
  \begin{center}
    \caption{\ke{Mean test error rate on the CIFAR benchmark averaged over 5 runs, with different  confidence threshold \ryn{values,} $\xi$.  
    \ryn{Parentheses show the numbers} of the labeled samples.}}
    \label{tab:xi_values}
    \vspace{-2mm}
    \begin{tabular}{lllll}
      \toprule 
      \textbf{Dataset (Labels)} & $\xi=0.0$ & $\xi=0.4$ & $\xi=0.6$ & $\xi=0.8$  \\
      \midrule
      CIFAR-10 \;\;(1k)  & $16.49$ & $16.12$ & $15.92$ & $\mathbf{15.74}$\\
      CIFAR-100 (10k)    & $33.67$ & $\mathbf{33.08}$& $33.23$ & $33.54$ \\
      \bottomrule
    \end{tabular}
  \end{center}
  \vspace{-5mm}
\end{table}

\subsection{Domain Adaptation}
Domain adaptation aims to transfer knowledge learned from a labeled dataset to an unlabeled one.
French \etalke\,\cite{MTDA} modified Mean Teacher and Temporal Model to enable domain adaptation
and showed the effectiveness of the Teacher-Student structure. 
\ke{In this section, we \ryn{apply Dual Student for adapting the digit recognition model from USPS to MNIST and show}
%\yq{\sout{through an USPS to MNIST\,\cite{MNIST} domain adaptation task to}}
%\yq{to help adapt digit recognition model from USPS to MNIST and}
that it could be applied to this \yq{kind of task}
with great advantages over the EMA teacher \yq{based methods}. }
%\ke{\sout{Here we verify that our Dual Student can also be generalized to domain adaptation by a digital classification task from USPS to MNIST\,\cite{MNIST}}}.

Both USPS and MNIST are greyscale hand-written \ryn{number} dataset. USPS consists of 7,000 images of  
$16\times16$, and MNIST contains 60,000 images of $28\times28$. To match the image resolution, 
we resize all images from USPS to $28\times28$ by cubic spline interpolation. Fig.\;\ref{fig:domain} 
shows the domain difference between the two datasets. 
In our experiments, we set USPS as the source domain and MNIST as the target domain. 
% To preserve the domain diversity, only random noise is applied on samples 
% as augmentation. 
We compare our method with Mean Teacher, 
%\yq{\sout{the result supervised from the source domain, and the result supervised from the target domain}}
\yq{source domain (USPS) supervised model, and target domain (MNIST) supervised model}
(trained \ryn{on} 7k balanced labels). All experiments use a small architecture simplified 
from the above 13-layer CNN. More details \ryn{are available} in \ke{Appendix D (Supplementary)}.

\ryn{Fig.\;\ref{fig:transfer_curve} shows 
% the error rates in Table\;\ref{tab:domain adaptation} and 
the test accuracy versus the number of epochs. We can see that naively using supervision from USPS would result in overfitting.} 
Mean Teacher avoids it to some extent 
% via the consistency constraint 
and improves the top1 accuracy 
from 69.09\% to 80.41\%\ryn{, but} it overfits when the \ryn{number of training epochs} is large.
% since the teacher is no better than the student in the MINIST domain. 
Our Dual Student avoids overfitting and \ryn{boosts the accuracy} to 91.50\%, which is much closer to the result obtained by 
supervision from the target domain. 
% The results demonstrate the potential of Dual Student in domain adaptation.

\begin{figure}[t]
    \setlength{\abovecaptionskip}{-0.2cm}
    \begin{center}
       \includegraphics[width=0.95\linewidth]{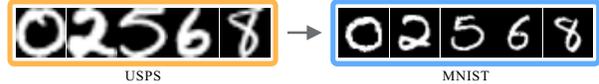}
    \end{center}
    \caption{Domain difference between USPS and MNIST. The \ryn{numbers} in USPS are in 
             bold font face and span all over the images without border.}
    \label{fig:domain}
    % \vspace{-4mm}
\end{figure}

\begin{figure}[t]
    \setlength{\abovecaptionskip}{-0.2cm}
    \begin{center}
       \includegraphics[width=0.75\linewidth]{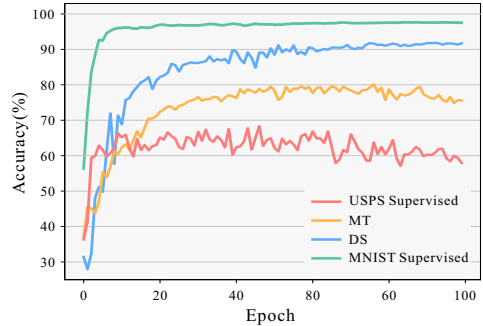}
    \end{center}
    \caption{Test curves of domain adaptation from USPS to MNIST \ryn{versus the number of} epochs. 
             Dual Student avoids overfitting and improves the result remarkably.}
    \label{fig:transfer_curve}
    % \vspace{-5mm}
\end{figure}

% \begin{table}[t]
%   \begin{center}
%     \caption{Error rate on domain adaptation from USPS to MNIST averaged over 5 runs.}
%     \label{tab:domain adaptation}
%     \begin{tabular}{lll}
%       \toprule 
%       \textbf{Model}            & 7k labels-top1            & 7k labels-top5          \\
%       \midrule
%       USPS Supervised           & $30.91 \pm 1.63$ & $6.33 \pm 0.90$\\
%       MT\,\cite{MeanTeacher}     & $17.59 \pm 1.67$ & $1.82 \pm 0.14$\\
%       DS                       & $8.95  \pm 0.56$ & $0.32 \pm 0.07$\\
%       MNIST Supervised          & $1.59  \pm 0.07$ & $0.04 \pm 0.03$\\
%       \bottomrule
%     \end{tabular}
%   \end{center}
% %   \vspace{-0.5cm}
% \end{table} 

%-------------------------------------------------------------------------
\section{Conclusion}
In this paper, we have studied the coupling effect of the existing Teacher-Student 
methods and \ryn{shown} that it \yq{sets} a performance bottleneck \yq{for} the structure. 
We \ryn{have} proposed a new structure, Dual Student, to break limits of the EMA teacher\ryn{, and a novel stabilization constraint, which provides an effective way to train independent models (either with the same architecture or not)}. 
\ke{The stabilization constraint}
% \rynq{It (*** Are you referring to Dual Student? If so, I suggest to change "It" to "Dual Student". ***)} 
is bidirectional overall but is unidirectional for each {\it stable sample}. 
The improved performance is notable across datasets and tasks. Besides,
\ryn{we have also discussed two variants of Dual Student,} with even better results. \ke{However, our method still shares similar limitations as existing methods, \ryn{e.g., increased memory usage during training and performance degradation on increasing number of labels}.}
In the future, we \ryn{plan} \ke{to address these \yq{issues} and extend} our structure 
to other applications.

\section*{Appendix A: Convergence of the EMA}
In our paper, we state that the EMA teacher is coupled with the student in the existing 
Teacher-Student methods.
We provide below a formal proposition for this statement and a simple proof. 
% In our paper, we denote that ``The EMA of a converging sequence converges to the same limit as the sequence''. 
% Formally, we prove the following equivalent description:
% We prove it in this section.
% In this section, we will prove the following statement:
% First, we define the problem as:

\begin{proposition}
Given a sequence $\{\,s_t\,\}_{t \in \mathbb{N}} \subseteq \mathbb{R}^{m}$ 
and let $s'_t = \alpha\,s'_{t-1} + (1-\alpha)\,s_t$, where $0<\alpha<1$, $t \in \mathbb{N}$, $s'_0 \in \mathbb{R}^{m}$. 
If $\{\,s_t\,\}_{t \in \mathbb{N}}$ converges to $S \in \mathbb{R}^{m}$,
then $\{\,s'_t\,\}_{t \in \mathbb{N}}$ converges to $S$ as well.
\end{proposition}

\begin{proof}
By the definition of convergence, if $\{\,s_t\,\}_{t \in \mathbb{N}}$ converges to $S$, we have:  $\forall \epsilon > 0$, $\exists T \in \mathbb{N}$ such that $\forall t > T$, $|s_t - S| < \epsilon$. 
First, when $t > T$, by the formula of the sum of a finite geometric series, we rewrite $S$ and $s'_t$ as:
\begin{equation}
\label{eq:1}
\begin{split}
S &= (1-\alpha)\,\frac{1-\alpha^{t-T}}{1-\alpha}\,S + \alpha^{t-T}\,S \\
  &= (1-\alpha)\sum_{i=T+1}^{t} \alpha^{t-i} S + \alpha^{t-T} S \,,\\
% \end{equation}
% %
% \begin{equation}
% \label{splitbn}
% \begin{split}
s'_t & = \alpha^{t}s'_0 +  (1-\alpha) \sum_{i=1}^{t}    \alpha^{t-i} s_{i} \\
    & = \alpha^{t}s'_0 + (1-\alpha)\sum_{i=1}^{T} \alpha^{t-i} s_{i}
     + (1-\alpha) \sum_{i=T+1}^{t} \alpha^{t-i} s_{i}\,.\\
\end{split}
\end{equation}
Since $T$ is finite, $\alpha^{T} s'_{0}$ and $\sum_{i=1}^{T} \alpha^{T-i} s_{i}$ are bounded. Thus, $\exists C\in \mathbb{R}^{+}$ such that: 
\begin{equation} \nonumber
|\alpha^{T} s'_0 + (1-\alpha)\sum_{i=1}^{T} \alpha^{T-i} s_{i}| < C\,.
\end{equation}
Since $0<\alpha<1$, we have $\lim_{t\to\infty}\alpha^t=0$. % $\displaystyle{\lim_{t\to\infty}\alpha^t=0}$, 
Thus, $\exists T'>0$ such that $\forall t>T', \alpha^{t} < min\{\,\frac{\epsilon}{C},\, \frac{\epsilon}{|S|}\,\}$. 
Then, after substituting Eq.\,\ref{eq:1} into $|s'_t-S|$ and applying the Triangular Inequality, we have: 
\begin{equation}
\label{eq:2}
\begin{split}
|s'_t-S| &\leq |\alpha^{t} s'_0 + (1-\alpha)\sum_{i=1}^{T} \alpha^{t-i} s_{i}| \\
         &+ |(1-\alpha) \sum_{i=T+1}^{t} \alpha^{t-i} (s_i - S)|
         + |\alpha^{t-T}S|\,. 
\end{split}
\end{equation}
Then $\forall t > (T+T')$, we have: 
%
% \begin{equation}
% \label{eq:3}
% \begin{split}
%      & |\alpha^{t} s'_0 + (1-\alpha)\sum_{i=1}^{T} \alpha^{t-i} s_{i}|  \\
%      =\;&\alpha^{t-T}\,|\alpha^{T} s'_0 + (1-\alpha)\sum_{i=1}^{T} \alpha^{T-i} s_{i}|
%     < \frac{\epsilon}{C}\,C < \epsilon \\
%      & |(1-\alpha) \sum_{i=T+1}^{t} \alpha^{t-i} (s_i - S)| \\
%      \leq\;&(1-\alpha) \sum_{i=T+1}^{t} \alpha^{t-i} |s_i - S| = (1-\alpha^{t-T})\,\epsilon < \epsilon \\
% % |(1-\alpha) \sum_{i=T+1}^{t} \alpha^{t-i} (s_i - S)| < (1-\alpha) \sum_{i=T+1}^{t} \alpha^{t-i} \frac{\epsilon}{3} <\frac{\epsilon}{3}
%      & |\alpha^{t-T}S| < \frac{\epsilon}{|S|} |S| < \epsilon \\
% \end{split}
% \end{equation} 
\begin{equation}
\label{eq:3}
\begin{split}
     & |\alpha^{t} s'_0 + (1-\alpha)\sum_{i=1}^{T} \alpha^{t-i} s_{i}|  \\
     =\;&\alpha^{t-T}\,|\alpha^{T} s'_0 + (1-\alpha)\sum_{i=1}^{T} \alpha^{T-i} s_{i}|
     < \frac{\epsilon}{C}\,C < \epsilon\,, \\
\end{split}
\end{equation} 
\begin{equation}
\label{eq:4}
\begin{split}
     & |(1-\alpha) \sum_{i=T+1}^{t} \alpha^{t-i} (s_i - S)| \\
     \leq\;&(1-\alpha) \sum_{i=T+1}^{t} \alpha^{t-i} |s_i - S| = (1-\alpha^{t-T})\,\epsilon < \epsilon\,, \\
\end{split}
\end{equation} 
\begin{equation}
\label{eq:5}
% \begin{split}
     |\alpha^{t-T}S| < \frac{\epsilon}{|S|} |S| < \epsilon\,. \quad\quad\quad\quad\quad\quad\quad\quad\quad\;\,
% \end{split}
\end{equation} 
Combining Eq.\;\ref{eq:2},\;\ref{eq:3},\;\ref{eq:4},\;\ref{eq:5}, we have $|s'_t - S| < 3\epsilon$, $\forall t > (T+T')$,
% Summarizing, we proved that $\forall \epsilon > 0$, $\exists (T+T')$ such that $\forall t > (T+T')$, $|s'_t-S| < \epsilon$, 
i.e., $\{s'_t\}_{y \in \mathbb{N}}$ converges to $S$.
\end{proof}

\section*{Appendix B: Model Architectures}
The model architecture used in our CIFAR-10, CIFAR-100, and SVHN experiments is the 
13-layer convolutional network (13-layer CNN), which is the same as previous works 
\cite{MeanTeacher, Temporal_Pi, FastSWA, SmoothNeighbor, DeepCoTrain}. We implement 
it following FastSWA\,\cite{FastSWA} for comparison. Table\,\ref{tab:13_cnn} describes 
its architecture in details. 
For ImageNet experiments, we use a 50-layer ResNeXt\,\cite{ResNeXt} architecture, which 
includes 3+4+6+3 residual blocks and uses the group convolution with 32 groups.
% Its channels are split into 32 groups. 

\section*{Appendix C: Semi-supervised Learning Setups}
In our work, all experiments use the SGD optimizer with the nesterov momentum set to $0.9$. 
The learning rate is adjusted by the function $\gamma = \gamma_0 * (0.5 + \cos((t - 1) * \pi / N))$, 
where $t$ is the current training step, $N$ is the total number of steps, and $\gamma_0$ is the initial learning 
rate. 
We present the settings of the experiments on each dataset as follows.

\textbf{CIFAR-10: } 
On CIFAR-10, we set the batch size to 100 and half of the samples in each batch are labeled. 
The initial learning rate is $0.1$. The weight decay is $1e^{-4}$. 
For the stabilization constraint, we set its coefficient $\lambda_{2} = 100$ 
and ramp it up in the first 5 epochs.
We set $\lambda_{1} = 10$.
The confidence threshold for the {\it stable samples} is $0.8$.

\textbf{CIFAR-100: } 
On CIFAR-100, each minibatch contains 128 samples, including 31 labeled samples. 
We set the initial learning rate to $0.2$ and the weight decay to $2e^{-4}$. 
The confidence threshold is $\xi = 0.4$. Other hyperparameters are the same as CIFAR-10.

\begin{table}[t]
  \begin{center}
    \caption{The 13-layer CNN for our SSL experiments.}
    \vspace{-0.2cm}

    \label{tab:13_cnn}
    \begin{tabular}{ll}
      \toprule 
      \textbf{Layer}        &\textbf{Details}                                             \\
      \midrule
      input                 & $32 \times 32 \times 3$ RGB image                           \\ 
      augmentation          & random translation, horizontal flip                         \\
    %   noise                 & gaussian noise $\zeta$ = $0.15$                                     \\ 
      \hline
      convolution           & $128$, $3 \times 3$, pad = {\it same}, LReLU $\alpha$ = $0.1$    \\
      convolution           & $128$, $3 \times 3$, pad = {\it same}, LReLU $\alpha$ = $0.1$    \\
      convolution           & $128$, $3 \times 3$, pad = {\it same}, LReLU $\alpha$ = $0.1$    \\
      pooling               & $2 \times 2$, type = {\it maxpool}                          \\
      dropout               & $p$ = $0.5$                                                 \\

      convolution           & $256$, $3 \times 3$, pad = {\it same}, LReLU $\alpha$ = $0.1$    \\
      convolution           & $256$, $3 \times 3$, pad = {\it same}, LReLU $\alpha$ = $0.1$    \\
      convolution           & $256$, $3 \times 3$, pad = {\it same}, LReLU $\alpha$ = $0.1$    \\
      pooling               & $2 \times 2$, type = {\it maxpool}                          \\
      dropout               & $p$ = $0.5$                                                 \\
      
      convolution           & $512$, $3 \times 3$, pad = {\it valid}, LReLU $\alpha$ = $0.1$   \\
      convolution           & $256$, $1 \times 1$, LReLU $\alpha$ = $0.1$                      \\
      convolution           & $128$, $1 \times 1$, LReLU $\alpha$ = $0.1$                      \\
      pooling               & $6 \times 6 \Rightarrow 1 \times 1$, type = {\it avgpool}   \\
      dense                 & $128 \Rightarrow 10$, softmax                               \\
      \bottomrule
    \end{tabular}
  \end{center}
  \vspace{-0.5cm}
\end{table}

\begin{table}[t]
  \begin{center}
    \caption{The small CNN for domain adaptation.}
    \vspace{-0.2cm}

    \label{tab:3_cnn}
    \begin{tabular}{ll}
      \toprule 
      \textbf{Layer}        &\textbf{Details}                                               \\
      \midrule
      input                 & $28 \times 28 \times 1$ Gray image                            \\ 
      augmentation                 & gaussian noise $\zeta$ = $0.15$                              \\ 
      \hline
      convolution           & $16$, $3 \times 3$, pad = {\it same}, LReLU $\alpha$ = $0.1$  \\
      pooling               & $2 \times 2$, type = {\it maxpool}                            \\

      convolution           & $32$, $3 \times 3$, pad = {\it same}, LReLU $\alpha$ = $0.1$  \\
      pooling               & $2 \times 2$, type = {\it maxpool}                            \\
      dropout               & $p$ = $0.5$                                                   \\
      
      convolution           & $32$, $3 \times 3$, pad = {\it same}, LReLU $\alpha$ = $0.1$  \\
      pooling               & $6 \times 6 \Rightarrow 1 \times 1$, type = {\it avgpool}     \\
      dense                 & $32 \Rightarrow 10$, softmax                                  \\
      \bottomrule
    \end{tabular}
  \end{center}
  \vspace{-0.5cm}
\end{table}

\textbf{SVHN: } 
The batch size on SVHN is 100, and each minibatch contains only 10 labeled samples.
The initial learning rate is $0.1$, and the weight decay is $1e^{-4}$. 
The stabilization constraint is scaled by $10$ (ramp up in 5 epochs). 
We use the confidence threshold $\xi = 0.8$.

\textbf{ImageNet: } 
We validate our method on ImageNet by the ResNeXt-50 architecture on 8 GPUs with batch size $320$ and half of the batch are labeled samples. 
Each sample is augmented following \cite{SaE_Net} and is resized to $224 \times 224$.
We warm-up the learning rate from $0.08$ to $0.2$ in the first $2$ epochs. 
The model is trained for $60$ epochs with the weight decay set to $5e^{-5}$, 
the stabilization constraint coefficient set to $1000$, and a small confidence
threshold of $0.01$.

% \ke{The training time of Dual Student is longer than Mean Teacher, since
% we need to optimize another student.}

\section*{Appendix D: Domain Adaptation Setups}
We design a small convolutional network for the domain adaptation from USPS (source domain) 
to MNIST (target domain). 
The structure is shown in Table \ref{tab:3_cnn}. We train all experiments for 100 epochs 
by the SGD optimizer with the nesterov momentum set to $0.9$ and the weight decay set to $1e^{-4}$. 
The learning rate declines from $0.1$ to $0$ by a cosine adjustment.
Each batch includes 256 samples while 32 of them are labeled. We randomly extract 
7000 balanced samples from MNIST for target-supervised experiments, and other experiments 
are done by using the training set of USPS.
The coefficient of the stabilization constraint is $\lambda_{2} = 1.0$. We also ramp it up in 
the first 5 epochs. 
The confidence threshold is $\xi = 0.6$. We discover that the input noise with 
$\zeta = 0.15$ is vital for the Mean Teacher but not for our method in this experiment.

{\small
\bibliographystyle{ieee_fullname}
\bibliography{egbib}
}

\end{document}